\documentclass[11pt]{article}
\usepackage{lmodern}
\usepackage[utf8]{inputenc}
\usepackage[english]{babel}
\usepackage[babel]{csquotes}
\usepackage[T1]{fontenc}
\usepackage[font=small,skip=0pt]{caption}
\usepackage[font=small]{caption} 
\usepackage[normalem]{ulem}
\usepackage{hyperref,url,booktabs,amsfonts,nicefrac,microtype,tikz,pgfplots,xcolor,guit,mdframed,morefloats,sidecap,lipsum,amsmath,amssymb,amsthm,mathtools,booktabs,tabularx,graphicx,subfig,xcolor,subfig,ucs,hyperref,booktabs,multirow,algorithm,algorithmic,amsfonts,multicol,ulem,enumitem,thmtools,thm-restate,amsmath,amssymb,amsthm,mathtools,appendix,bbm,multirow}
\usepackage[top=2.5cm, bottom=3.5cm, left=2.5cm, right=2.5cm]{geometry}
\usepackage{todonotes}
\pgfplotsset{compat=1.14}

\allowdisplaybreaks
\usepackage{colortbl}
\definecolor{bblue}{HTML}{4F81BD}
\definecolor{rred}{HTML}{C0504D}
\definecolor{ggreen}{HTML}{9BBB59}
\definecolor{ppurple}{HTML}{9F4C7C}
\definecolor{turquoise}{HTML}{008B8B}
\definecolor{darkgray}{HTML}{A9A9A9}
\definecolor{saddlebrown}{HTML}{FFE4B5}
\definecolor{Gray}{gray}{0.9}
\definecolor{LightCyan}{rgb}{0.88,1,1}

\newcolumntype{g}{>{\columncolor{Gray}}c}

\newcommand{\trans}{^{\scriptscriptstyle \top}}

\newcommand{\Exp}{\mathbb{E}}

\newcommand{\X}{{\cal X}}
\newcommand{\E}{{\cal E}}
\newcommand{\Y}{{\cal Y}}

\newcommand{\la}{{\langle}}
\newcommand{\ra}{{\rangle}}

\newcommand{\argmin}{\operatornamewithlimits{argmin}}

\declaretheorem[name=Theorem,refname=Thm.]{theorem}

\declaretheorem[name=Lemma,sibling=theorem]{lemma}

\declaretheorem[name=Corollary,refname=Cor.,sibling=theorem]{corollary}

\makeatletter
\renewenvironment{proof}[1][\proofname]{\par
 \pushQED{\qed}%
 \normalfont \topsep6\p@\@plus6\p@\relax
 \trivlist
 \item[\hskip\labelsep
 \bfseries
 #1\@addpunct{.}]\ignorespaces
}{%
 \popQED\endtrivlist\@endpefalse
}
\makeatother
\def\eop{$\rule{1.3ex}{1.3ex}$}
\renewcommand\qedsymbol\eop
\makeatletter
\renewenvironment{proof}[1][\proofname]{\par
 \pushQED{\qed}%
 \normalfont \topsep6\p@\@plus6\p@\relax
 \trivlist
 \item[\hskip\labelsep
 \bfseries
 #1\@addpunct{.}]\ignorespaces
}{%
 \popQED\endtrivlist\@endpefalse
}
\makeatother
\def\eop{$\rule{1.3ex}{1.3ex}$}
\renewcommand\qedsymbol\eop
\newcommand*{\shortautoref}[1]{%
 \begingroup
 \def\sectionautorefname{Sec.}%
 \autoref{#1}%
 \endgroup
}

\newcommand{\R}{{\mathbb R}}

\newcommand{\beq}{\begin{equation}}
\newcommand{\eeq}{\end{equation}} 
\newcommand{\bea}{\begin{eqnarray}}
\newcommand{\eea}{\end{eqnarray}}

\newcommand{\trace}{{\rm tr}}

\def\boldf#1{\hbox{\rlap{$#1$}\kern.4pt{$#1$}}}

\newcommand{\cR}{{\cal R}}

\newcommand{\x}{{\bf x}} 
 
\newcommand{\y}{{\bf y}} 
\newcommand{\z}{{\bf z}} 
\newcommand{\HH}{{\mathbb{H}}}

\title{\sffamily\LARGE Learning Fair and Transferable Representations\footnote{This work was supported by the Amazon AWS Machine Learning Research Award.}}

\author{Luca Oneto \\
\small Universit\'a di Pisa\\
\small Department of Informatics, 56127 Pisa, Italy\\
\small {\em luca.oneto@gmail.com} \\[3mm] 
Michele Donini \\
\small Amazon Web Services, Seattle, WA 98109, USA \\
\small {\em donini@amazon.com}\\[3mm]
Andreas Maurer  \\
\small Adalbertstr 55, D-80799 Munich, Germany \\
\small {\em am@andreas-maurer.eu} \\[3mm]
Massimiliano Pontil \\
\small Istituto Italiano di Tecnologia,
16163 Genoa, Italy \\
\small {\em massimiliano.pontil@iit.it} \\
\small and \\
\small University College London, Department of Computer Science \\
\small London WC1E 6BT, UK 
}

\date{\today}

\begin{document}
\maketitle
\begin{abstract}
Developing learning methods which do not discriminate subgroups in the population is a central goal of algorithmic fairness.
One way to reach this goal is by modifying the data representation in order to meet certain fairness constraints.
In this work we measure fairness according to demographic parity.
This requires the probability of the possible model decisions to be independent of the sensitive information.
We argue that the goal of imposing demographic parity can be substantially facilitated within a multitask learning setting.
We leverage task similarities by encouraging a shared fair representation across the tasks via low rank matrix factorization.
We derive learning bounds establishing that the learned representation transfers well to novel tasks both in terms of prediction performance and fairness metrics.
We present experiments on three real world datasets, showing that the proposed method outperforms state-of-the-art approaches by a significant margin.
\end{abstract}
\newpage
\section{Introduction}
\footnotetext[1]{Computer Science Department, University College London, WC1E 6BT London, United Kingdom}\footnotetext[2]{Computational Statistics and Machine Learning - Istituto Italiano di Tecnologia, 16100 Genova, Italy}\footnotetext[3]{Electrial and Electronics Engineering Department, Imperial College London, SW7 2BT, United Kingdom.}
During the last decade, the widespread distribution of automatic systems for decision making is raising concerns about their potential for unfair behaviour~\cite{barocas2016big,raji2019actionable,buolamwini2018gender}.
As a consequence machine learning models are often required to meet fairness requirements, ensuring the correction and limitation of -- for example -- racist or sexist decisions.

In literature, it is possible to find a plethora of different methods to generate fair models with respect to one or more sensitive attributes (e.g.~gender, ethnic group, age).
These methods can be mainly divided in three families: (i) methods in the first family change a pre-trained model in order to make it more fair (while trying to maintain the classification performance)~\cite{feldman2015certifying,hardt2016equality,pleiss2017fairness}; (ii) in the second family, we can find methods that enforce fairness directly during the training phase, e.g.
~\cite{zafar2019fairness,donini2018empirical,zafar2017fairness,agarwal2018reductions}; (iii) the third family of methods implements fairness by modifying the data representation, and then employs standard machine learning methods~\cite{zemel2013learning,calmon2017optimized}.

All methods in the previous families have in common the goal of creating a fair model from scratch on the specific task at hand.
This solution may work well in specific cases, but in a large number of real world applications, using the same model (or at least part of it) over different tasks is helpful if not mandatory.
For example, it is common to perform a fine tuning over pre-trained models~\cite{donahue2014decaf}, keeping fixed the internal representation.
Indeed, most modern machine learning frameworks (especially the deep learning ones) offer a set of pre-trained models that are distributed in so-called model zoos\footnote{See for example the Caffe Model Zoo: \url{github.com/BVLC/caffe/wiki/Model-Zoo}}.
Unfortunately, fine tuning pre-trained models on novel previously unseen tasks could lead to an unexpected unfairness behaviour, even starting from an apparently fair model for previous tasks (e.g.
discriminatory transfer~\cite{discriminative_transfer} or negative legacy~\cite{kamishima2012fairness}), due to missing generalization guarantees concerning the fairness property of the model.

In order to overcome the above problem, in this paper we embrace the framework of multitask learning.
We aim to leverage task similarity in order to learn a fair representation that provably generalizes well to unseen tasks.
By this we mean that when the representation is used to learn novel tasks, it is guaranteed to learn a model that has both a small error and meets the fairness requirement.
We measure fairness according to demographic parity~\cite{calders2009building} (for an extended analysis of the different fairness definitions see~\cite{fairness2018verma,zafar2019fairness}).
It requires the probability of possible model decisions to be independent of the sensitive information.
We argue that multitask methods based on low rank matrix factorization are well suited to learn a shared fair representation according to demographic parity.
We show theoretically that the learned representation transfers to novel tasks both in terms of prediction performance and fairness metrics.
Other papers in literature already pursued a similar goal~\cite{beutel2017data,edwards2015censoring,louizos2015variational,madras2018learning,mcnamara2017provably,mcnamara2019costs,wang2018invariant}.
They mainly rely on generating a model acting randomly when the internal representation is exploited to predict the sensitive variable.
No actual constraint is imposed directly on the internal representation, but only over the output of the model.


The main contribution of this paper is to augment multitask learning methods based on low rank matrix factorization by imposing a fairness constraint directly on the representation factor matrix.
We show empirically and theoretically, via learning bounds, that by imposing the fairness constraint within the multitask learning method, the learned representation can be used to train new models over different (new and possibly unseen) tasks, maintaining the desiderata of an accurate and fair model.
Our learning bound improves over previous bounds for learning-to-learn and by being fully data dependent can be used to evaluate the transfer capability of the learned representation.

The paper is organized in the following manner.
In~\shortautoref{sec:related}, we discuss previous related work aimed at learning fair representations.
In~\shortautoref{sec:3}, we introduce the proposed method.
In~\shortautoref{sec:4}, we study the generalization properties of the method, embracing the framework of learning-to-learn.
In~\shortautoref{sec:5}, we experimentally compare the proposed method against different baselines and state-of-the-art approaches on three real world datasets.
Finally, in~\shortautoref{sec:6} we discuss directions of future research.
\section{Related work}
\label{sec:related}
Let us consider a composition of models $f(g(x))$ where $x \in \R^d$ is a vector of raw features (an element of the input space), $g: \mathbb{R}^d \rightarrow \mathbb{R}^r$ is a function mapping the input space into a new one, that we refer to as the representation.
In other words, the function $g$ synthesizes the information needed to solve a particular task (or a set of tasks) by learning a function $f$, chosen from a set of possible functions.

In this work -- and more generally in the current literature~\cite{beutel2017data,edwards2015censoring,louizos2015variational,madras2018learning,mcnamara2017provably,mcnamara2019costs,wang2018invariant,johansson2016learning,zemel2013learning} -- with fair representation we refer to the concept of learning a representation function $g$, which does not discriminate subgroups in the data.
Namely, $g$ is conditionally independent of subgroup membership.
This approach is different from most commonly used approaches~\cite{donini2018empirical,hardt2016equality,zafar2017fairness}, in which the focus is to solve a task (or a set of tasks) without discriminating subgroups in the data, regardless of the fairness of the representation itself.
That is, in the previously mentioned work a fair model $f:\mathbb{R}^r \rightarrow \mathbb{R}$ is learned directly from the raw data, without performing any explicit representation extraction. 

In particular, in~\cite{beutel2017data,edwards2015censoring,louizos2015variational,madras2018learning,mcnamara2017provably,mcnamara2019costs,wang2018invariant}, the authors propose different neural networks architectures together with modified learning strategies able to learn a representation that obscures or removes the sensitive variable.
In the general case, all these methods have an input, a target variable (i.e.~the task at hand) and a binary sensitive variable.
The objective is to learn a representation that: (i) preserves information about the input space; (ii) is useful for predicting the target; (iii) is approximately independent of the sensitive variable.
In practice, these methods pursue the goal of making the generated model act randomly when the internal representation is exploited to predict the sensitive variable.
In this sense, no actual constraint is directly imposed on the internal representation, but only on the output of the model.

In~\cite{johansson2016learning}, instead, the authors show how to formulate the problem of counterfactual inference as a domain adaptation problem, and more specifically a covariate shift problem~\cite{quionero2009dataset}.
The authors derive two new families of representation algorithms for counterfactual inference.
The first one is based on linear models and variable selection, and the other one on deep learning. The authors show that learning representations that encourage similarity (i.e.~balance) between the
treatment and control populations leads to better counterfactual inference; this is in contrast to many methods which attempt to create balance by re-weighting samples.

Finally, in~\cite{zemel2013learning}, the authors learn a representation of the data that is a probability distribution over clusters where learning the cluster of a datapoint contains no-information about the sensitive variable, namely fair clustering.
In this sense, the clustering is learned to be fair and also discriminative for the prediction task at hand.
\section{Method}
\label{sec:3}
In this section, we present our method to learn a shared fair representation from multiple tasks. We consider $T$ supervised learning tasks (each could be a binary classification or regression problem).
Each task $t \in \{1,\dots,T\}$ is identified by a probability distribution $\mu_t$ on $\mathcal{X} \times \mathcal{S} \times \mathcal{Y}$, where $\mathcal{X} \subset \R^d$ is the set of non-sensitive input variables, $\mathcal{S} = \{1,2\}$ is the set of values of a binary sensitive variable\footnote{Our method naturally extends to multiple sensitive variables but for ease of presentation we consider only the binary case.} and $\mathcal{Y}$ is the output space which is either $\{-1,1\}$ for binary classification or $\mathcal{Y} \subset \R$ for regression.
We let ${\bf z}_{t} = ({x}_{t,i},s_{t,i},y_{t,i})_{i=1}^{m} \in (\mathcal{X} \times \mathcal{S} \times \mathcal{Y})^m$ be the training sequence for task $t$, which is sampled independently from $\mu_t$.
The goal is to learn a predictive model $f_t: \mathcal{X} \times \mathcal{S} \rightarrow \mathcal{Y}$ for each task $t \in \{1,\dots,T\}$.

Depending on the application at hand, the model may include (i.e. $f: \mathcal{X} \times \mathcal{S} \rightarrow \mathcal{Y}$ or not (i.e. $f: \mathcal{X}  \rightarrow \mathcal{Y}$) the sensitive feature in its functional form. In the following we consider the case that the functions $f_t$ are linear, and to simplify the presentation we consider the case that $s$ is not included in the functional form, that is, $f_t(x) = \la {w}_t, x \ra$, where ${w}_t \in \R^d$ is a vector of parameters. 
The case in which both $x$ and $s$ are used as predictors is obtained by adding two more components to $x$, representing the one-hot encoding of $s$, and letting ${w}_t \in \R^{d+2}$.

A general multitask learning formulation (MTL) is based on minimizing the multitask empirical error plus a regularization term which leverages similarities between the tasks. A natural choice for the regularizer which is considered in this paper is given by the trace norm, namely the sum of the singular values of the matrix $W = [{w}_1 \cdots {w}_T] \in \mathbb{R}^{d\times T}$.
It is well know, that this problem is equivalent to the matrix factorization problem,
\begin{align}
\min_{A,B} \quad & \frac{1}{Tm}\sum_{t=1}^T \sum_{i=1}^{m} \big( y_{t,i}-\la {b}_t, A\trans{x}_{t,i}\ra \big)^2 +\frac{\lambda}{2} \big(\|A\|_F^2 + \|B\|_F^2\big)
\label{eq:2}
\end{align}
where $A=[a_1\dots a_r] \in \mathbb{R}^{d \times r}$ and $B =[b_1\dots b_T]\in \mathbb{R}^{r \times T}$ and $\| \cdot \|_F$ is the Frobenius norm, see e.g.~\cite{srebro} and references therein.
Here $r \in \mathbb{N}$ is the number of factors, that is the upper bound on the rank of $W=AB$. 
If $r \geq \min(d,T)$ then Problem~\eqref{eq:2} is equivalent to trace norm regularization~\cite{AEP2008}, see e.g.~\cite{CSP} and references therein\footnote{If $r <\min(d,T)$ then Problem~\eqref{eq:2} is equivalent to trace norm regularization plus a rank constraint.}.
We follow the formulation of Eq.~\eqref{eq:2} since it can easily be solved by gradient descent or alternate minimization as we discuss next.
Once the problem is solved the estimated parameters of the function $w_t$ for the tasks' linear models are simply computed as $w_t=Ab_t$.
We also note that for simplicity the problem is stated with the square loss function, but our observations extended to the general case of proper convex loss functions.

Note that the method can be interpreted as a $2$-layer network with linear activation functions. Indeed, the matrix $A\trans$ applied to an input vector $x \in \mathbb{R}^d$ induces the linear representation $A\trans {x} = ( {a}_1\trans {x}, \cdots, {a}_r\trans {x})\trans$. 
We would like this representation to be fair w.r.t.~the sensitive feature. 
Specifically, we require that each component of the representation vector satisfies the demographic parity constraint~\cite{gajane2017formalizing,fairness2018verma} on each task. This means that, for every measurable subset $C \subset \mathbb{R}^{r}$, and for every $t \in \{1,\dots,T\}$, we require that 
\begin{align}
\label{eq:comp}
\mathbb{P}( A\trans {x}_t \in C~ | s = 1) = \mathbb{P}(A\trans {x}_t \in C ~| s = 2) 
\end{align}
that is the two conditional distributions are the same.
We relax this constraint by requiring, for every $t \in \{1,\dots,T\}$, that both distributions have the same mean. Furthermore, we compute the means from empirical data. For each training sequence ${\bf z} \in (\X \times \Y)^T$ and $s \in {\cal S}$, we use the notation $I_s({\bf z}) = \{(x_i,y_i) : s_i = s\}$, define the empirical conditional means
\begin{align}
c({\bf z}) = \frac{1}{|I_1({\bf z})|} \sum_{i\in I_1({\bf z})} {x}_i - 
\frac{1}{|I_2({\bf z})|} \sum_{i\in I_2({\bf z})} {x}_i
\end{align}
and then relax the constraint of Eq.~\eqref{eq:comp} to
\begin{align}
\label{eq:epsT}
A\trans {c}({\bf z}_t)
= 0.
\end{align}
This is a crude approximation since it corresponds to requiring the first order moment of the two distribution to be the same. However, as we shall see, it works well in practice and has the major advantage of turning a non-convex constraint in a convex one. We note that a similar approximation has been considered in \cite{oneto2019general} in the case of fair regression, and reported to be empirically effective.

Based on the above reasoning, we propose to learn a fair linear representation as a solution to the optimization problem
\begin{align}
\label{prob:mt}
\min_{A,B} \quad & \frac{1}{Tm}\sum_{t=1}^T \sum_{i=1}^{m} \big( y_{t,i}-\la b_t, A\trans{x}_{t,i}\ra \big)^2 +\frac{\lambda}{2} (\|A\|_F^2 + \|B\|_F^2) \\
& A\trans {c}(\z_t)
= 0, \quad
t \in \{1,\dots,T\}.
\nonumber
\end{align}
where we used the shorthand notation $c_t = c(\z_t)
$. There are many methods to tackle Problem~\eqref{prob:mt}.
A natural approach is based on alternate minimization.
We discuss the main steps below. Let ${y}_t = [y_{t,1}, \dots, y_{t,m}]\trans$, the vector formed by the outputs of task $t$, and let $X_t = [{x}_{t,1}\trans,\dots, {x}_{t,m}\trans]\trans$, the data matrix for task $t$.

When we regard $A$ as fixed and solve w.r.t.
$B$, then Problem~\eqref{prob:mt} can be reformulated as
\begin{align}
\min_{B} \quad & 
\left\| 
\begin{bmatrix}
{y}_1 \\
\vdots\\
{y}_T \\
\end{bmatrix}
- 
\begin{bmatrix}
X_1 A & 0 & \cdots & 0 \\
\vdots & \vdots & \vdots & \vdots \\
0 & \cdots & 0 & X_T A \\
\end{bmatrix}
\begin{bmatrix}
{b}_1 \\
\vdots\\
{b}_T \\
\end{bmatrix}
\right\|^2 
+ \lambda 
\left\| 
\begin{bmatrix}
{b}_1 \\
\vdots \\
{b}_T \\
\end{bmatrix}
\right\|^2 
\end{align}
which can be easily solved. In particular note that the problem decouples across the tasks, and each task specific problem amounts running ridge regression on the data transformed by the representation matrix $A\trans$. When instead $B$ is fixed and we solve w.r.t.
$A$, Problem~\eqref{prob:mt} can be reformulated as
\begin{align*}
\min_{A} & 
\left\| 
\begin{bmatrix}
{y}_1 \\
\vdots\\
{y}_T \\
\end{bmatrix}
- 
\begin{bmatrix}
b_{1,1} X_1 & \cdots & b_{1,r} X_1 \\
& \vdots \\
b_{t,1} X_T & \cdots & b_{t,r} X_T \\
\end{bmatrix} 
\begin{bmatrix}
{a}_1 \\
\vdots \\
{a}_r \\
\end{bmatrix} 
\right\|^2 
+ \lambda 
\left\| 
\begin{bmatrix}
{a}_1 \\
\vdots \\
{a}_r \\
\end{bmatrix}
\right\|^2 ~~\text{s.t.} ~~
\begin{bmatrix}
{a}_1^T \\
\vdots \\
{a}_r^T \\
\end{bmatrix}
\circ
\begin{bmatrix}
{c}_1, \dots, {c}_T
\end{bmatrix} = {0} 
\end{align*} 
where $\circ$ is the Kronecker product for partitioned tensors (or Tracy-Singh product).
Consequently by alternating minimization we can solve the original problem.
Note also that we may relax the equality constraint as $\frac{1}{T}\sum_{t=1}^T \|A\trans c({\bf z}_t)\|^2 \leq \epsilon$, where $\epsilon$ is some tolerance parameter.
In fact, this may be required when the vectors $c({\bf z}_t)$ span the all input space.
In this case we may also add a soft constraint in the regularizer.

We conclude this  section by noting that if demographic parity is satisfied at the representation level, that is,~Eq.~\eqref{eq:comp} holds true, then every model built from such representation will satisfy demographic parity as well.
Likewise if the representation satisfies the convex relaxation of Eq.~\eqref{eq:epsT}, then it will also hold that $\la w_t,c({\bf z}_t)\ra = \la b_t, A\trans c({\bf z}_t)\ra =0$, that is the task weight vectors will satisfy the first order moment approximation of demographic parity.
More importantly, as we will show in the next section, if the tasks are randomly observed, then demographic parity will also be satisfied on future tasks with high probability. In this sense our method can be interpreted as learning a fair transferable representation.
\section{Learning bound}
\label{sec:4}
In this section, we study the learning ability of the proposed method. We consider the setting of learning-to-learn~\cite{Baxter2000}, in which the training tasks (and their corresponding datasets) used to find a fair data representation are regarded as random variables from a meta-distribution.
The learned representation matrix $A$ is then transferred to a novel task, by applying ridge regression on the task dataset, in which the input $x$ is transformed as $A\trans x$. 
In
\cite{Mau9} a learning bound is presented, linking the average risk of the method over tasks from the meta-distribution (the so-called transfer risk) to the multi-task empirical error on the training tasks.
This result quantifies the good performance of the representation learning method when the number of tasks grow and the data distribution on the raw input data is intrinsically high dimensional (hence learning is difficult without representation learning). We extend this analysis to the setting of algorithmic fairness, in which the performance of the algorithm is evaluated both relative to risk and the fairness constraint. We show that both quantities can be bounded by their empirical counterparts evaluated on the training tasks. 

To present our result we introduce some more notation. We let $\E_\mu(w)$ and $\E_{\z}(w)$ be the expected and empirical errors of a weight vector $w$, that is
$$
\E_\mu(w) = \mathbb{E}_{(x,y) \sim \mu} [(y-\la w,x\ra)^2],~~~~~~~~\E_\z(w) =  \frac{1}{m} \sum_{i=1}^m (y_i-\la w,x_i\ra)^2.
$$
Furthermore, for every  matrix $A \in \mathbb{R}^{d \times r}$ and for every data sample $\z = (x_i,y)_{i=1}^m$, we define $b_A(\z) = \argmin_{b \in \mathbb{R}^r} \frac{1}{m} \sum_{i=1}^m (y_i-\la b,A\trans x_i\ra)^2 +  \lambda \|b\|^2$
be the minimizer of ridge regression with modified data representation, that is
where ``$^+$'' is the pseudo-inverse operation.

\begin{theorem}
\label{prop:1}
Let $A$ be the representation learned by solving Problem~\eqref{eq:2} and renormalized so that $\|A\|_F = 1$. 
Let tasks $\mu_1,\dots,\mu_T$ be independently sampled from a meta-distribution $\rho$, and let $\z_t$ be sampled from $\mu_t^m$ for $t\in\{1,\dots,T\}$. Assume that the input marginal distribution of random tasks from $\rho$ is supported on the unit sphere and that the outputs are in the interval $[-1,1]$, almost surely. Let $r = \min(d,T)$. Then, for any $\delta\in(0,1]$ it holds with probability at least $1-\delta$ in the drawing of the datasets ${\bf z}_1,\dots,{\bf z}_T$, that
%
\begin{align*}
 \Exp_{\mu\sim\rho}\Exp_{{\bf z}\sim\mu^m} ~ \cR_\mu\big(w_A({\bf z})\big) - \frac{1}{T} \sum_{t=1}^T {\cal R}_{{\bf z}_t}(w_A({\bf z}_t)) \leq 
 \frac{4}{\lambda}\sqrt{\frac{ \|{\hat C}\|_\infty}{m}} + \frac{24}{\lambda m}\sqrt{\frac{ \ln \frac{8 mT}{\delta}}{T}} \\
  + \frac{14}{\lambda} \sqrt{\frac{\ln(mT)\|{\hat C\|_\infty}}{T}}
+ \sqrt{\frac{2\ln  \frac{4}{\delta} }{T}}
\end{align*}
%
and
\begin{equation*}
 \Exp_{\mu\sim\rho}\Exp_{{\mathcal {\bf z}\sim\mu^m}} \| A c({\bf z})\|^2 - \frac{1}{T} \sum_{t=1}^T\| A c({\bf z}_t) \|^2 \leq 96 \frac{ \ln \frac{8r^2}{\delta}}{T}+ 6 \sqrt{\frac{{\|\hat \Sigma}\|_\infty \ln \frac{8 r^2}{\delta}}{T}}.
\end{equation*}
\end{theorem}
\begin{proof}
Let $D=\frac{1}{\lambda} A\trans A$. Note that algorithm $\z \mapsto w_D(\z) = Ab_A(\z)$ is equivalent to run regularized least squares on the original dataset, constraining the paramter vector $w$ to be in the range of $D$ and using the regularizer $w\trans D^+ w$, where ``$^+$'' denote the pseudo-inverse. The first claim follows from Theorem \ref{Theoremprincipal} stated in the appendix, with ${\cal D} = \{D\succeq 0, \trace D \leq 1/\lambda\}$, noting that the algorithm has kernel stability $2$, the function $M(K) = 2 K + 1$, $\|D\|_\infty \leq \|D\|_1 = 1/\lambda$. 
We then use the first inequality in Corollary \ref{Corollary random operators} in the appendix to upper bound $\sqrt{\|C\|}$ by $\sqrt{\|\hat C\|} + 6 \sqrt{(\ln (4 mT)/\delta)/(mT)}$ and a union bound.
%

To prove the second claim we note that
\begin{align}
\frac{1}{T} \sum_{t=1}^T \|Ac({\bf z}_t)\|^2 = \trace D {\hat \Sigma},~~~{\hat \Sigma} = \frac{1}{T} \sum_{t=1}^T c({\bf z}_t) \otimes c({\bf z}_t)
\label{eq:empcov}
\end{align}
and similarly
\begin{align}
\Exp_{\mu\sim\rho}\Exp_{{\bf z}\sim\mu^m} \|Ac({\bf z})\|^2 = \trace D {\Sigma},~~~{\Sigma} = 
\Exp_{\mu\sim\rho} \Exp_{{\bf z}\sim\mu^m} c({\bf z}) \otimes c({\bf z}).
\end{align}
Then
\begin{align*}
\Exp_{\mu\sim\rho}\Exp_{{\bf z}\sim\mu^m} \|Ac({\bf z})\|^2 - \frac{1}{T}\sum_{t=1}^T \|Ac({\bf z}_t)\|^2 = \trace D ({\Sigma}-{\hat \Sigma}) \leq \|D\|_1 \|{\Sigma}-{\hat \Sigma}\|_\infty = \|{\Sigma}-{\hat \Sigma}\|_\infty.
\end{align*}
The second inequality then follows immediately from inequality \eqref{Second corollary inequality} in \autoref{Corollary random operators}, with $N=T$ and $A_t = (1/4)c(\bf{z}_t)\otimes c(\bf{z}_t)$.
\end{proof}

We make some remarks on the above result:
\begin{enumerate}
\item The first bound in Theorem~\ref{prop:1} improves Theorem~2 in~\cite{Mau9}.
The improvement is due to the introduction of the empirical total covariance in the second term in the RHS of the inequality. The result in \cite{Mau9} instead contains the term $\sqrt{1/T}$, which can be considerably larger when the raw input is distributed on a high dimensional manifold. 
\item The bounds in Theorem~\ref{prop:1} can be extended to hold with variable sample size per task.
In order to simplify the presentation, we assume that all datasets are composed of the same number of points $m$.
The general setting can be addressed by letting the sample size be a random variable and introducing the slightly different definition of the transfer risk in which we also take the expectation w.r.t.
the sample size.
\item The hyperparameter $\lambda$ is regarded as fixed in the analysis.
In practice it will be chosen by cross-validation as in our experiments below.
\item The bound on fairness measure contains two terms in the right hand side, in the spirit of Bernstein's inequality.
The slow term $O(1/\sqrt{T})$ contains the spectral norm of the covariance of difference of means across the sensitive groups.
Notice that 
$\|\Sigma\|_\infty \leq 1$ but it can be much smaller when the means are close to each other, that is, when the original representation is already approximately fair.
\end{enumerate}
\section{Experiments}
\label{sec:5}
In this section, we compare the proposed method against different baselines and state-of-the-art-methods.

\noindent {\bf Settings.} In order to better understand the performance of the proposed method we performed two sets of experiments.

In the first set (Table~\ref{tab:statistics1new}) we compare the following methods: (a) Unconstrained single task learning (STL), (b) Fair constrained STL, (c) Unconstrained MTL, (d) Fair constrained MTL, that is the proposed method. 
We test each method either on the same tasks exploited during the training phase, or on novel tasks.
Furthermore, we consider both the case where the sensitive feature is present, and not in the functional form of the model (i.e.~the sensitive feature is known or not in the testing phase).

In the second set of experiments (Table~\ref{tab:statistics2new}) we compare, in the same setting that we just described, (a) Standard MTL with the fairness constraints on the outputs (M1), (b) feed-forward neural network (FFNN) with linear activation and the fair shared representation method presented in~\cite{madras2018learning} (M2), (c) FFNN with linear activation by exploiting a fair shared representation as presented in~\cite{edwards2015censoring} (M3), (d) Fair constrained MTL (Our Method).
We used linear activation functions in FFNN for fair comparison, since the proposed method learns linear models. 

Concerning the experiments on the same task setting, we train the model with all the tasks and then we measure results on an independent test set of the same tasks.
In the case of novel task experiments, we train the model with all the tasks minus one (randomly selected).
Then, we fix the representation found by our method and we use a subset of the data (70\%) for the excluded task to train the last layer, maintaining fixed the representation layer.
Finally, we used the remaining data (30\%) of the novel task as test set, measuring both error and fairness measure. 

We repeated all the experiments both with and without the sensitive feature in the functional form of the model.
We validated the hyperparameters using a grid search with $\lambda \in \{10^{-6.0}, 10^{-5.8}, \cdots, 10^{+4.0} \}$ and $r \in \{ 2^j d \,\, | \,\, j = -4, -3, \dots, 10 \}$, following the validation procedure in~\cite{donini2018empirical}.
Specifically, in the first step, the classical 10-fold CV error for each of the combination of the hyperparameters is computed.
In the second step, we shortlist all the hyperparameters’ combinations with error close to the best one (in our case, above 90\% of the smallest error).
Finally, from this list, we select the hyperparameters with the smallest fairness risk.
Concerning the error (ERR) we used mean average precision error as the performance index, and concerning the fairness of our model (FAIR), we compute the the difference of demographic parity as $\frac{1}{|\mathcal{Y}|} \sum_{y \in \mathcal{Y}} |P(f(x) = y | s = 1) - P(f(x)=y | s = 2)|$, since in our datasets the output space is finite.
For all the experiments, we report performance over $30$ repetitions with the corresponding standard deviation.

\noindent {\bf Datasets.} In our comparisons we used three datasets.
The first one is the School data set~\cite{goldstein1991multilevel} -- made available by the Inner London Education Authority (ILEA) -- formed by examination records from 139 secondary schools in years 1985, 1986 and 1987.
It is a random 50\% sample with 15362 students.
Each task in this setting is to predict exam scores for students in one school, based on eight inputs.
The first four inputs (year of the exam, gender, VR band and ethnic group) are student-dependent, the next four (percentage of students eligible for free school meals, percentage of students in VR band one, school gender (mixed or single-gender and school denomination) are school-dependent.
The categorical variables (year, ethnic group and school denomination) were split up in one-hot variables, one for each category, making a new total of 16 student-dependent inputs, and six school-dependent inputs.
We scaled each covariate and output to have zero mean and unit variance.
The sensitive attribute is the gender of the student.
The second dataset we propose has been collected at the University of Genoa\footnote{The data and the research are related to the project DROP@UNIGE of the University of Genoa.} (UNIGE) and is also exploited in~\cite{oneto2019general}.
This dataset is a proprietary and highly sensitive dataset containing all the data about the past and present students enrolled at the UNIGE.
In this study we take into consideration students who enrolled, in the academic year (a.y.) 2017-2018.
The dataset contains $5000$ instances, each one described by $35$ attributes (both numeric and categorical) about ethnicity, gender, financial status, and previous school experience.
The scope is to predict the grades at the end of the first semester being fair with respect to the gender of the student.
The sensitive attribute is the gender of the student.
Finally, the third dataset is Movielens~\cite{harper2016movielens}.
Specifically, we considered Movielens 100k (ml100k), which consists of ratings (1 to 5) provided by 943 users for a set of 1682 movies, with a total of 100,000 ratings available.
Additional features for each movie, such as the year of release or its genre, are provided.
The sensitive attribute is the gender of the user.

\begin{table}
\caption{Comparison between  the following method: (a) Unconstrained single task learning (STL), (b) Fair constrained STL, (c) Unconstrained MTL, (d) Fair constrained MTL, that is the proposed method.}
\centering
\scriptsize
\setlength{\tabcolsep}{.1cm}
\renewcommand{\arraystretch}{1.4}
\begin{tabular}{clggcc|ggcc}
\toprule
& & \multicolumn{2}{g}{STL - UnCons} & \multicolumn{2}{c|}{STL - Cons} & \multicolumn{2}{g}{MTL - UnCons} & \multicolumn{2}{c}{MTL - Cons} \\
& Dataset & ERR & FAIR & ERR & FAIR & ERR & FAIR & ERR & FAIR \\ 
\midrule
\parbox[t]{2mm}{\multirow{8}{*}{\rotatebox[origin=c]{90}{{\bf Same Tasks}}}}
& \multicolumn{9}{c}{Sensitive feature not in the functional form of the model}\\
& School 
& $15.30 {\pm} 0.60$ & $0.110 {\pm} 0.005$ 
& $16.37 {\pm} 0.34$ & $0.044 {\pm} 0.003$ 
& $10.71 {\pm} 0.57$ & $0.077 {\pm} 0.003$
& $11.78 {\pm} 0.75$ & $0.011 {\pm} 0.000$ \\
 & UNIGE 
& $19.50 {\pm} 0.94$ & $0.100 {\pm} 0.006$
& $20.87 {\pm} 1.16$ & $0.040 {\pm} 0.002$
& $13.65 {\pm} 0.47$ & $0.070 {\pm} 0.003$
& $15.02 {\pm} 0.54$ & $0.010 {\pm} 0.001$ \\
 & Movielens 
& $30.30 {\pm} 1.98$ & $0.160 {\pm} 0.008$
& $32.42 {\pm} 1.14$ & $0.048 {\pm} 0.002$
& $15.15 {\pm} 0.60$ & $0.112 {\pm} 0.008$
& $17.27 {\pm} 0.76$ & $0.000 {\pm} 0.000$ \\ 
& \multicolumn{9}{c}{Sensitive feature in the functional form of the model}\\
& School
& $14.23 {\pm} 0.70$ & $0.118 {\pm} 0.006$
& $15.30 {\pm} 0.81$ & $0.052 {\pm} 0.003$
& $ 9.64 {\pm} 0.40$ & $0.085 {\pm} 0.004$
& $10.71 {\pm} 0.52$ & $0.019 {\pm} 0.001$ \\
& UNIGE
& $18.13 {\pm} 0.83$ & $0.107 {\pm} 0.005$
& $19.50 {\pm} 0.71$ & $0.047 {\pm} 0.003$
& $12.29 {\pm} 0.67$ & $0.077 {\pm} 0.004$
& $13.65 {\pm} 0.82$ & $0.017 {\pm} 0.001$ \\
& Movielens
& $28.18 {\pm} 1.35$ & $0.171 {\pm} 0.010$
& $30.30 {\pm} 1.28$ & $0.059 {\pm} 0.002$
& $13.03 {\pm} 0.47$ & $0.123 {\pm} 0.007$
& $15.15 {\pm} 0.73$ & $0.011 {\pm} 0.001$ \\ 
\midrule
\parbox[t]{2mm}{\multirow{8}{*}{\rotatebox[origin=c]{90}{{\bf New Tasks}}}}
& \multicolumn{9}{c}{Sensitive feature not in the functional form of the model}\\
& School
& $18.36 {\pm} 1.12$ & $0.121 {\pm} 0.007$
& $19.43 {\pm} 0.80$ & $0.055 {\pm} 0.003$
& $13.77 {\pm} 0.52$ & $0.088 {\pm} 0.003$
& $14.84 {\pm} 0.74$ & $0.022 {\pm} 0.001$ \\
& UNIGE
& $21.45 {\pm} 1.16$ & $0.105 {\pm} 0.006$
& $22.82 {\pm} 1.22$ & $0.045 {\pm} 0.002$
& $15.60 {\pm} 0.83$ & $0.075 {\pm} 0.003$
& $16.97 {\pm} 0.70$ & $0.015 {\pm} 0.001$ \\
& Movielens
& $33.33 {\pm} 2.14$ & $0.176 {\pm} 0.009$
& $35.45 {\pm} 1.84$ & $0.064 {\pm} 0.004$
& $18.18 {\pm} 0.76$ & $0.128 {\pm} 0.007$
& $20.30 {\pm} 1.18$ & $0.016 {\pm} 0.001$ \\ 
& \multicolumn{9}{c}{Sensitive feature in the functional form of the model}\\
& School
& $17.29 {\pm} 0.73$ & $0.129 {\pm} 0.007$
& $18.36 {\pm} 0.88$ & $0.063 {\pm} 0.004$
& $12.70 {\pm} 0.50$ & $0.096 {\pm} 0.005$
& $13.77 {\pm} 0.76$ & $0.030 {\pm} 0.002$ \\
& UNIGE
& $20.08 {\pm} 1.21$ & $0.112 {\pm} 0.005$
& $21.45 {\pm} 1.04$ & $0.052 {\pm} 0.002$
& $14.23 {\pm} 0.67$ & $0.082 {\pm} 0.001$
& $15.60 {\pm} 0.61$ & $0.022 {\pm} 0.001$ \\
& Movielens
& $31.21 {\pm} 1.63$ & $0.187 {\pm} 0.007$
& $33.33 {\pm} 1.28$ & $0.075 {\pm} 0.004$
& $16.06 {\pm} 0.92$ & $0.139 {\pm} 0.011$
& $18.18 {\pm} 0.79$ & $0.027 {\pm} 0.001$ \\ 
\bottomrule
\end{tabular}
\label{tab:statistics1new}
\end{table}
\begin{table}
\caption{Comparison of the following methods: (M1) Standard MTL with the fairness constraints on the outputs, (M2) FFNN with linear activation and the fair shared representation method presented in~\cite{madras2018learning}, (M3) FFNN with fair shared representation~\cite{edwards2015censoring}, (M4) Fair constrained MTL (Our Method).}
\centering
\scriptsize
\setlength{\tabcolsep}{.1cm}
\renewcommand{\arraystretch}{1.4}
\begin{tabular}{clggccggcc}
\toprule
& & \multicolumn{2}{g}{M1} & \multicolumn{2}{c}{M2} & \multicolumn{2}{g}{M3} & \multicolumn{2}{c}{M4 (OURS)} \\
& Dataset & ERR & FAIR & ERR & FAIR & ERR & FAIR & ERR & FAIR \\ 
\midrule
\parbox[t]{2mm}{\multirow{8}{*}{\rotatebox[origin=c]{90}{{\bf Same Tasks}}}}
& \multicolumn{9}{c}{Sensitive feature not in the functional form of the model}\\
& School
& $12.34 {\pm} 0.75$ & $0.013 {\pm} 0.001$
& $13.44 {\pm} 1.04$ & $0.017 {\pm} 0.002$
& $12.93 {\pm} 0.79$ & $0.018 {\pm} 0.002$
& $11.78 {\pm} 0.75$ & $0.011 {\pm} 0.000$ \\
& UNIGE
& $18.12 {\pm} 0.98$ & $0.012 {\pm} 0.001$
& $21.23 {\pm} 1.34$ & $0.021 {\pm} 0.004$
& $26.19 {\pm} 1.76$ & $0.027 {\pm} 0.004$
& $15.02 {\pm} 0.54$ & $0.010 {\pm} 0.001$ \\
& Movielens
& $17.12 {\pm} 0.65$ & $0.009 {\pm} 0.000$
& $19.21 {\pm} 0.87$ & $0.014 {\pm} 0.002$
& $18.01 {\pm} 0.76$ & $0.012 {\pm} 0.002$
& $17.27 {\pm} 0.76$ & $0.007 {\pm} 0.000$ \\
& \multicolumn{9}{c}{Sensitive feature in the functional form of the model}\\
& School
& $11.01 {\pm} 0.91$ & $0.020 {\pm} 0.001$
& $12.01 {\pm} 1.01$ & $0.022 {\pm} 0.002$
& $13.31 {\pm} 1.23$ & $0.025 {\pm} 0.002$
& $10.71 {\pm} 0.52$ & $0.019 {\pm} 0.001$ \\
& UNIGE
& $13.75 {\pm} 0.82$ & $0.017 {\pm} 0.001$
& $20.13 {\pm} 1.24$ & $0.029 {\pm} 0.005$
& $25.92 {\pm} 1.76$ & $0.032 {\pm} 0.006$
& $13.65 {\pm} 0.82$ & $0.017 {\pm} 0.001$ \\
& Movielens
& $15.65 {\pm} 0.73$ & $0.010 {\pm} 0.001$
& $18.97 {\pm} 0.67$ & $0.017 {\pm} 0.004$
& $17.11 {\pm} 0.78$ & $0.015 {\pm} 0.003$
& $15.15 {\pm} 0.73$ & $0.011 {\pm} 0.001$ \\
\midrule
\parbox[t]{2mm}{\multirow{8}{*}{\rotatebox[origin=c]{90}{{\bf New Tasks}}}}
& \multicolumn{9}{c}{Sensitive feature not in the functional form of the model}\\
& School
& $15.64 {\pm} 0.79$ & $0.032 {\pm} 0.002$
& $16.43 {\pm} 1.11$ & $0.044 {\pm} 0.004$
& $17.21 {\pm} 1.32$ & $0.041 {\pm} 0.004$
& $14.84 {\pm} 0.74$ & $0.022 {\pm} 0.001$ \\
& UNIGE
& $16.21 {\pm} 0.97$ & $0.021 {\pm} 0.002$
& $21.98 {\pm} 1.47$ & $0.029 {\pm} 0.004$
& $27.31 {\pm} 1.23$ & $0.033 {\pm} 0.005$
& $16.97 {\pm} 0.70$ & $0.015 {\pm} 0.001$ \\
& Movielens
& $19.20 {\pm} 1.35$ & $0.025 {\pm} 0.002$
& $21.21 {\pm} 1.35$ & $0.031 {\pm} 0.004$
& $20.12 {\pm} 1.43$ & $0.030 {\pm} 0.003$
& $20.30 {\pm} 1.18$ & $0.016 {\pm} 0.001$ \\
& \multicolumn{9}{c}{Sensitive feature in the functional form of the model}\\
& School
& $14.72 {\pm} 0.87$ & $0.038 {\pm} 0.002$
& $18.02 {\pm} 1.07$ & $0.042 {\pm} 0.003$
& $17.92 {\pm} 0.87$ & $0.056 {\pm} 0.003$
& $13.77 {\pm} 0.76$ & $0.030 {\pm} 0.002$ \\
& UNIGE
& $15.89 {\pm} 0.68$ & $0.029 {\pm} 0.002$
& $19.21 {\pm} 1.04$ & $0.035 {\pm} 0.005$
& $25.87 {\pm} 1.23$ & $0.038 {\pm} 0.006$
& $15.60 {\pm} 0.61$ & $0.022 {\pm} 0.001$ \\
& Movielens
& $19.98 {\pm} 0.74$ & $0.038 {\pm} 0.002$
& $20.12 {\pm} 1.12$ & $0.037 {\pm} 0.003$
& $19.93 {\pm} 1.53$ & $0.038 {\pm} 0.004$
& $18.18 {\pm} 0.79$ & $0.027 {\pm} 0.001$ \\
\bottomrule
\end{tabular}
\label{tab:statistics2new}
\end{table}
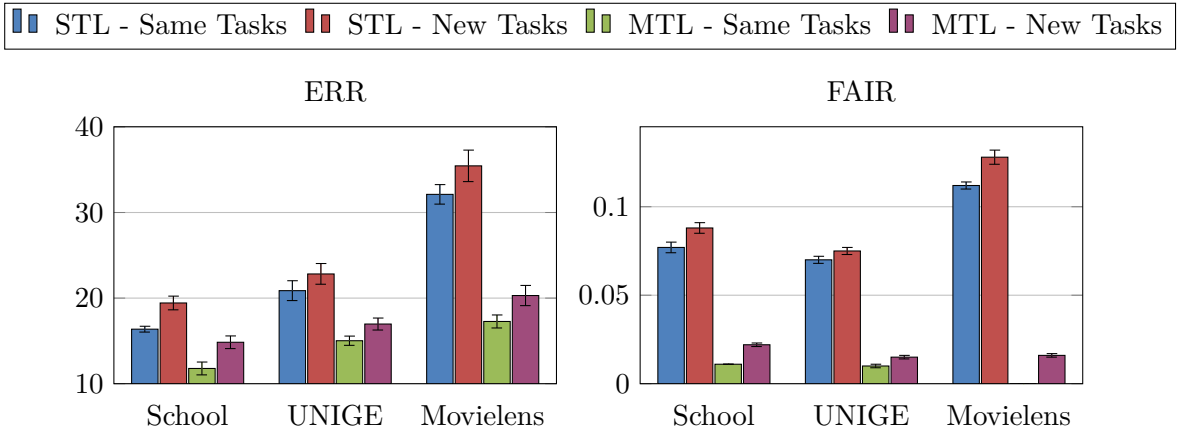
\begin{figure}[t]
    \centering
\begin{tikzpicture}
    \begin{axis}[
        width  = 0.45*\textwidth,
        height = 5cm,
        major x tick style = transparent,
        ybar=2*\pgflinewidth,
        bar width=10pt,
        ymajorgrids = true,
        title = {ERR},
        symbolic x coords={School,UNIGE,Movielens},
        xtick = data,
        scaled y ticks = false,
        enlarge x limits=0.25,
        ymin=10,
        legend cell align=left,
        legend columns=-1,
        legend style={
                at={(2.4,1.3)},
                anchor=south east,
                column sep=1ex
        }
    ]
    
    \addplot[style={fill=bblue,mark=none},error bars/.cd, y dir=both, y explicit] coordinates {(School,16.37)  += (0,0.34) -= (0,0.34)
    (UNIGE,20.87)  += (0,1.16) -= (0,1.16)
    (Movielens,32.12)  += (0,1.14) -= (0,1.14) }; 
    \addplot[style={fill=rred,mark=none},error bars/.cd, y dir=both, y explicit] coordinates 
    {(School,19.43)   += (0,0.8) -= (0,0.8)
    (UNIGE,22.82)  += (0,1.22) -= (0,1.2)
    (Movielens,35.45)  += (0,1.84) -= (0,1.84) }; 
    \addplot[style={fill=ggreen,mark=none},error bars/.cd, y dir=both, y explicit] coordinates {(School,11.78) += (0,0.75) -= (0,0.75)
    (UNIGE,15.02)  += (0,0.54) -= (0,0.54)
    (Movielens,17.27) += (0,0.76) -= (0,0.76)}; 
    \addplot[style={fill=ppurple,mark=none},error bars/.cd, y dir=both, y explicit] coordinates {(School,14.84)  += (0,0.74) -= (0,0.74)
    (UNIGE,16.97)  += (0,0.70) -= (0,0.70)
    (Movielens,20.3) += (0,1.18) -= (0,1.18)}; 

    \legend{STL - Same Tasks,STL - New Tasks,MTL - Same Tasks,MTL - New Tasks}
    \end{axis} \hspace{7cm}
        \begin{axis}[
        width  = 0.45*\textwidth,
        height = 5cm,
        major x tick style = transparent,
        ybar=2*\pgflinewidth,
        bar width=10pt,
        ymajorgrids = true,
        title = {FAIR},
        symbolic x coords={School,UNIGE,Movielens},
        xtick = data,
        scaled y ticks = false,
        enlarge x limits=0.25,
        ymin=0,
yticklabel style={
        /pgf/number format/fixed,
        /pgf/number format/precision=5
},
scaled y ticks=false,
        legend cell align=left,
        legend style={
                at={(1,1.05)},
                anchor=south east,
                column sep=1ex
        }
    ]
    
    \addplot[style={fill=bblue,mark=none},error bars/.cd, y dir=both, y explicit] coordinates {(School,0.077)  += (0,0.003) -= (0,0.003)
    (UNIGE,0.07)  += (0,0.002) -= (0,0.002)
    (Movielens,0.112)  += (0,0.002) -= (0,0.002)}; 
    \addplot[style={fill=rred,mark=none},error bars/.cd, y dir=both, y explicit] coordinates {(School,0.088)   += (0,0.003) -= (0,0.003)
    (UNIGE,0.075)   += (0,0.002) -= (0,0.002)
    (Movielens,0.128)  += (0,0.004) -= (0,0.004)}; 
    \addplot[style={fill=ggreen,mark=none},error bars/.cd, y dir=both, y explicit] coordinates {(School,0.011)   += (0,0.0) -= (0,0.0)
    (UNIGE,0.01)   += (0,0.001) -= (0,0.001)
    (Movielens,0.0)  += (0,0.0) -= (0,0.0)}; 
    \addplot[style={fill=ppurple,mark=none},error bars/.cd, y dir=both, y explicit] coordinates {(School,0.022)   += (0,0.001) -= (0,0.001)
    (UNIGE,0.015)   += (0,0.001) -= (0,0.001)
    (Movielens,0.016)  += (0,0.001) -= (0,0.001)}; 

    \end{axis}
\end{tikzpicture}
\caption{Graphical representation of the results in Table~\ref{tab:statistics1new}, when the sensitive feature is not included in the functional form of the model and the fairness constraint is active.}
\label{fig:stlmtlres}
\end{figure}
\begin{figure}[t]
    \centering
\hspace{-11cm}
\begin{tikzpicture}
    \begin{axis}[
        width  = 0.45*\textwidth,
        height = 5cm,
        major x tick style = transparent,
        ybar=2*\pgflinewidth,
        bar width=10pt,
        ymajorgrids = true,
        title = {ERR},
        symbolic x coords={School,UNIGE,Movielens},
        xtick = data,
        scaled y ticks = false,
        enlarge x limits=0.25,
        ymin=10,
        ymax=30,
        legend cell align=left,
        legend columns=-1,
        legend style={
                at={(2.3,1.3)},
                anchor=south east,
                column sep=1ex
        }
    ]
    
    \addplot[style={fill=turquoise,mark=none},error bars/.cd, y dir=both, y explicit] coordinates {(School,15.64)  += (0,0.79) -= (0,0.79)
    (UNIGE,16.21)  += (0,0.97) -= (0,0.97)
    (Movielens,19.20)  += (0,1.35) -= (0,1.35) }; 
    \addplot[style={fill=darkgray,mark=none},error bars/.cd, y dir=both, y explicit] coordinates 
    {(School,16.43)   += (0,1.11) -= (0,1.11)
    (UNIGE,21.98)  += (0,1.47) -= (0,1.47)
    (Movielens,21.21)  += (0,1.35) -= (0,1.35) }; 
    \addplot[style={fill=saddlebrown,mark=none},error bars/.cd, y dir=both, y explicit] coordinates {(School,17.21) += (0,1.32) -= (0,1.32)
    (UNIGE,27.31)  += (0,1.23) -= (0,1.23)
    (Movielens,20.12) += (0,1.43) -= (0,1.43)}; 
    \addplot[style={fill=ppurple,mark=none},error bars/.cd, y dir=both, y explicit] coordinates {(School,14.84)  += (0,0.74) -= (0,0.74)
    (UNIGE,16.97)  += (0,0.70) -= (0,0.70)
    (Movielens,20.3) += (0,1.18) -= (0,1.18)}; 


    \end{axis} \hspace{7cm}
        \begin{axis}[
        width  = 0.45*\textwidth,
        height = 5cm,
        major x tick style = transparent,
        ybar=2*\pgflinewidth,
        bar width=10pt,
        ymajorgrids = true,
        title = {FAIR},
        symbolic x coords={School,UNIGE,Movielens},
        xtick = data,
        scaled y ticks = false,
        enlarge x limits=0.25,
        ymin=0,
yticklabel style={
        /pgf/number format/fixed,
        /pgf/number format/precision=5
},
scaled y ticks=false,
        legend cell align=left,
        legend columns=-1,
        legend style={
                at={(0.45,1.3)},
                anchor=south east,
                column sep=1ex
        }
    ]
    
    \addplot[style={fill=turquoise,mark=none},error bars/.cd, y dir=both, y explicit] coordinates {(School,0.032)  += (0,0.002) -= (0,0.002)
    (UNIGE,0.021)  += (0,0.002) -= (0,0.002)
    (Movielens,0.025)  += (0,0.002) -= (0,0.002)}; 
    \addplot[style={fill=darkgray,mark=none},error bars/.cd, y dir=both, y explicit] coordinates {(School,0.044)   += (0,0.004) -= (0,0.004)
    (UNIGE,0.029)   += (0,0.004) -= (0,0.004)
    (Movielens,0.031)  += (0,0.004) -= (0,0.004)}; 
    \addplot[style={fill=saddlebrown,mark=none},error bars/.cd, y dir=both, y explicit] coordinates {(School,0.041)   += (0,0.004) -= (0,0.004)
    (UNIGE,0.033)   += (0,0.005) -= (0,0.005)
    (Movielens,0.030)  += (0,0.003) -= (0,0.003)}; 
    \addplot[style={fill=ppurple,mark=none},error bars/.cd, y dir=both, y explicit] coordinates {(School,0.022)   += (0,0.001) -= (0,0.001)
    (UNIGE,0.015)   += (0,0.001) -= (0,0.001)
    (Movielens,0.016)  += (0,0.001) -= (0,0.001)}; 

    \legend{M1, M2 ,M3, M4 (OURS)}
    \end{axis}
\end{tikzpicture}
\caption{Graphical representation of the results in Table~\ref{tab:statistics2new} for new tasks when the sensitive feature is not included in the functional form of the model.}
\label{fig:stlmtlres2}
\end{figure}
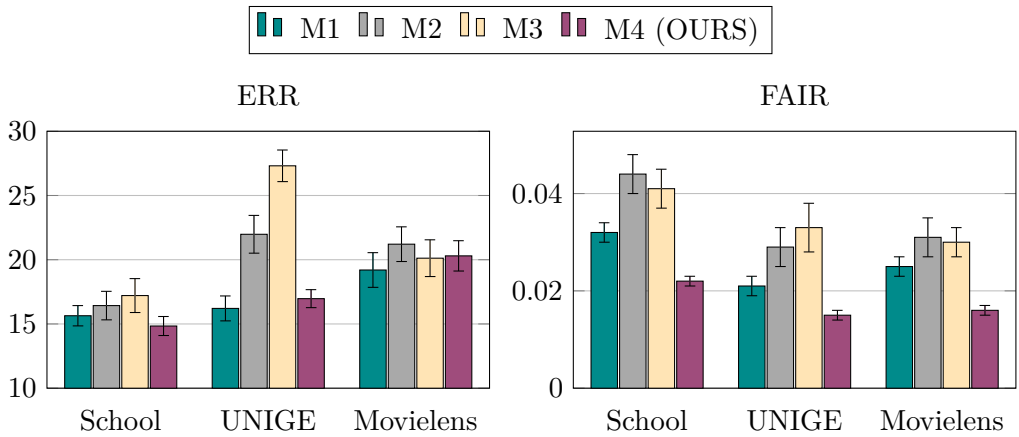
\noindent {\bf Discussion.} From our experimental results, different interesting aspects and comparisons can be extracted.
Firstly, the results in Table~\ref{tab:statistics1new} confirm the benefit of using a MTL approach in comparison to STL, in that accuracy has a significant improvement, both on same and novel tasks, thanks to the shared representation. Achieving less error has the positive side effect of producing a more fair model, even in the unconstrained case (i.e.~fair unaware). 

In the case of constrained methods, learning a fair shared representation slightly increases the final error but brings a large decrease of the fairness measure. From Table~\ref{tab:statistics1new}, we observe that this benefit is maintained also by tackling new and unseen (during the training of the shared representation) tasks.
In this sense, our method (constrained MTL) obtains the best performance among all the others.

In general, the same analysis of the results applies to both having and not having the sensitive feature in the functional form of the model.
In order to better interpret our results, and due to our higher interest in the case of a fair constrained model without the sensitive feature in the functional form of the model, we compared in Figure~\ref{fig:stlmtlres} the constrained STL approach to the constrained MTL approach (our method) both on the same and the novel tasks.
In this figure it is easier to note the benefits of our algorithm in decreasing both the error and the fairness measure.

Finally, we compared our method with three different state-of-the-art methods.
In Table~\ref{tab:statistics2new} and Figure~\ref{fig:stlmtlres2}, we show these results.
We note how our method, in all the possible settings, obtains better or comparable performance.
In fact, it is able to maintain a larger accuracy (comparable to the other methods) and simultaneously a smaller fairness risk.
\section{Conclusion}
\label{sec:6}
We have presented a method to learn a fair shared representation among different tasks in a MTL setting.
Our method is able to provide good generalization performance both in accuracy and fairness over novel and unseen tasks.
We studied the learning ability of our method in theory and we analyzed the performance over several experimental scenarios in practice.
The obtained results corroborate our theoretical findings and proved that our approach overcomes common benchmark algorithms and current state-of-the-art methods.
Our next step will be to study (explicit) fair representation learning in the context of shallow and deep neural networks, basically a generalization to the non-linear case of the proposed approach, with particular attention to the interpretability of the learned representation, in the context of transparency and trust of the machine learning model.
\bibliographystyle{plain}
\bibliography{biblio}
\newpage
\begin{appendices}
\section*{Appendix}

In this appendix, we first collect some tools used in our analysis. We then
present an improvement of Theorem 2 in [23], which is instrumental in the
proof of Theorem \ref{prop:1}, and explain how to pass to a fully data dependent
bound.

\section{Matrix concentration inequalities}

\begin{theorem}[Part (i) of Theorem 7 of \cite{Mau16}]
	\label{Theorem random operators} Let $%
	A_{1},\dots ,A_{N}$ be independent random operators on a Hilbert space
	satisfying $0\preceq A_{k}\preceq I$ and suppose that for some $d\in 
	\mathbb{N}
	$%
	\begin{equation}
	\dim {\rm Span}\left( {\rm Ran}\left( A_{1}\right) ,\dots ,{\rm Ran}\left( A_{N}\right)
	\right) \leq d  \label{Oliveira dimension condition}
	\end{equation}%
	almost surely. Then%
	\begin{equation}
	\Pr \left\{ \left\Vert \sum_{k}\left( A_{k}-\mathbb{E}A_{k}\right)
	\right\Vert _{\infty }>s\right\} \leq 4d^{2}\exp \left( \frac{-s^{2}}{%
		9\left\Vert \sum_{k}\mathbb{E}A_{k}\right\Vert _{\infty }+6s}\right) .
	\label{Matrix inequality from trace paper}
	\end{equation}%
\end{theorem}

\begin{corollary}
	\label{Corollary random operators}Under the conditions of the Theorem we have (i) with probability at least $%
	1-\delta $ that%
	\begin{equation}
	\sqrt{\left\Vert \mathbb{E}\sum A_{k}\right\Vert _{\infty }}\leq \sqrt{%
		\left\Vert \sum A_{k}\right\Vert _{\infty }}+6\sqrt{\ln \left( \frac{4d^{2}}{%
			\delta }\right) },  \label{First corollary inequality}
	\end{equation}%
	and (ii) with the same bound on the probability%
	\begin{equation}
	\left\Vert \sum_{k}\left( A_{k}-\mathbb{E}A_{k}\right) \right\Vert _{\infty
	}\leq 3\sqrt{\left\Vert \sum A_{k}\right\Vert _{\infty }\ln \frac{8d^{2}}{%
		\delta }}+24\ln \frac{8d^{2}}{\delta }  \label{Second corollary inequality}
\end{equation}
\end{corollary}

\begin{proof}
	Equating the RHS of (\ref{Matrix inequality from trace paper}) to $\delta $
	we obtain with probability at least $1-\delta $ that%
	\begin{equation}
	\left\Vert \sum_{k}\left( A_{k}-\mathbb{E}A_{k}\right) \right\Vert _{\infty
	}\leq \sqrt{9\left\Vert \sum_{k}\mathbb{E}A_{k}\right\Vert _{\infty }\ln \frac{4d^{2}}{\delta}}+6\ln \frac{ 4d^{2}}{\delta} 
\label{Third corollary inequality}
\end{equation}%
Denoting for brevity $a:=\left\Vert \sum_{k}\mathbb{E}A_{k}\right\Vert
_{\infty }$, $b:=\left\Vert \sum_{k}A_{k}\right\Vert _{\infty }$ and $%
c=4d^{2}$ we have with probability at least $1-\delta $%
\[
a-b\leq 2\sqrt{a}\sqrt{ \frac{9}{4} \ln \frac{c}{\delta} }+6\ln
\frac{c}{\delta} ,
\]%
or, subtracting $2\sqrt{a}\sqrt{\frac{9}{4} \ln \frac{c}{\delta}}$ and adding $b+\frac{9}{4} \ln \frac{c}{\delta} $,%
\begin{eqnarray*}
	\left( \sqrt{a}-\sqrt{\frac{9}{4} \ln \frac{c}{\delta} }\right)
	^{2} &=&a-2\sqrt{a}\sqrt{\frac{9}{4} \ln \frac{c}{\delta} }%
	+\frac{9}{4} \ln \frac{c}{\delta}  \\
	&\leq &b+\frac{9}{4} \ln \frac{c}{\delta} +6\ln \left( c/\delta
	\right) .
\end{eqnarray*}%
Taking the positive squareroot and adding $\sqrt{\frac{9}{4} \ln
	\frac{c}{\delta} }$ gives%
\begin{eqnarray*}
	\sqrt{a} &\leq &\sqrt{\frac{9}{4} \ln \frac{c}{\delta} }+\sqrt{%
		b+\frac{9}{4} \ln \frac{c}{\delta} +6\ln \frac{c}{\delta} } \\
	&\leq &\sqrt{b}+\sqrt{9\ln \frac{c}{\delta} }+\sqrt{6\ln \frac{c}{\delta} }\leq \sqrt{b}+6\sqrt{\ln \frac{c}{\delta} },
\end{eqnarray*}%
which is (\ref{First corollary inequality}). Part (ii) then follows from a
union bound of (\ref{Third corollary inequality}) with (\ref{First corollary
	inequality}).
\end{proof}

The second matrix concentration inequality we need is 

\begin{theorem}[Theorem 1.5 in \cite{Tropp12}]	\label{Theorem Oliveira Tropp} Let $%
	A_{1},\dots ,A_{N}$ be fixed matrices with dimension $d_{1}\times d_{2}$ and 
	$\gamma _{1},...,\gamma _{N}$ independent standard normal variables. Let $%
	\sigma ^{2}$ be the variance patrameter%
	\[
	\sigma ^{2}=\max \left\{ \left\Vert \sum_{k}A_{k}\right\Vert _{\infty
	},\left\Vert \sum_{k}A_{k}^{\ast }\right\Vert _{\infty }\right\} .
	\]%
	Then for $s>0$%
	\[
	\Pr \left\{ \left\Vert \sum_{k}\gamma _{k}A_{k}\right\Vert _{\infty
	}>s\right\} \leq \left( d_{1}+d_{2}\right) e^{-s^{2}/\left( 2\sigma
	^{2}\right) }.
\]%
\end{theorem}

\begin{corollary}
	\label{Corollary Oliveira Tropp}Under above assumptions, if $d_{1}+d_{2}\geq
	3$ then%
	\[
	\mathbb{E}\left\Vert \sum_{k}\gamma _{k}A_{k}\right\Vert _{\infty }\leq 
	\frac{5}{2}\sqrt{\sigma ^{2}\ln \left( d_{1}+d_{2}\right) }.
	\]
\end{corollary}

\begin{proof}
	Let $\delta =\sqrt{2\sigma ^{2}\ln \left( d_{1}+d_{2}\right) }$. Integration
	by parts gives%
	\begin{eqnarray*}
		\mathbb{E}\left\Vert \sum_{k}\gamma _{k}A_{k}\right\Vert _{\infty } &\leq
		&\delta +\left( d_{1}+d_{2}\right) \int_{\delta }^{\infty }e^{-t^{2}/2\sigma
			^{2}}dt \\
		&\leq &\delta +\left( d_{1}+d_{2}\right) \frac{\sigma ^{2}}{\delta }\exp
		\left( \frac{-\delta ^{2}}{2\sigma ^{2}}\right)  \\
		&\leq &\frac{5}{2}\sqrt{\sigma ^{2}\ln \left( d_{1}+d_{2}\right) }.
	\end{eqnarray*}%
\end{proof}

\section{Improved bound for the transfer risk}

We give an improvement of Theorem~2 in~\cite{Mau9} for the case of trace-norm regularization.

Most of the notation, definitions and assumptions are taken from~\cite{Mau9}, except that here we denote by $T$ the number of tasks and let $t \in \{1,\dots,T\}$ be the task index; these correspond to $n$ and $l\in \{1,\dots,n\}$ in \cite{Mau9}. We used this notation because it is common in the multitask literature. Our results are dimension free, so they hold in the general case that the input space is Hilbert space $\HH$ and $D$ a bounded linear operator on $\HH$. However to simplify the presentation here we take $\HH=\R^d$ and $D$ a $d \times d$ PSD matrix. We also use $R_\mu(w)$ and $R_{\z}(w)$ as the expected and empirical error of a weight vector $w$, that is
$$
\E_\mu(w) = \mathbb{E}_{(x,y) \sim \mu} [\ell(\la w,x\ra,y)],~~~~~~~~\E_\z(w) =  \frac{1}{m} \sum_{i=1}^m \ell(\la w,x_i\ra, y_i)
$$

For every PSD matrix $D$ we define the following quantities (see \cite[Secs.~2.2 \& 2.3]{Mau9})
$$
w(\x,\y) = \argmin_{w \in \mathbb{R}^d} \frac{1}{m} \sum_{i=1}^m \ell(\la w,x_i\ra, y_i) + \|w\|^2,~~~~~~~~w_D(\x,\y) = D^{\frac{1}{2}} w (D^{\frac{1}{2}} \x,\y).
$$
Note that the vector $w_D(\x,\y)$ corresponds to the minimizer of ridge regression with modified regularizer,
$$
w_D(\x,\y) = \argmin_{w \in {\rm Ran}(D)} \frac{1}{m} \sum_{i=1}^m \ell(\la w,x_i\ra, y_i) + w\trans D^+ w
$$
where $D^+$ is the pseudo-inverse of $D$.
%
\begin{theorem}
\label{Theoremprincipal}
Let ${\cal D}$ a subset of $d \times d$ PSD matrices. Suppose the algorithm $w$ is 1-bounded and has
kernel stability $L~\ $relative to the loss function $\ell $ and that for every $K<\infty $ there exists $M\left( K\right) $ such that for all $y \in \left[ 0,1\right] $ and for all $s,t\in \left[ -K,K\right] $ we have
\begin{align}
\ell \left( s,y\right) -\ell \left( t,y\right) \leq M\left( K\right)
\left\vert s-t\right\vert .
\end{align}
Then for every $\delta >0,$ with probability greater $1-\delta $ in the data $ \left( \mathbf{X},\mathbf{Y}\right) \sim \hat{\rho}^{T}$ we have for all $D\in \mathcal{D}$
\begin{align*}
\Exp_{\mu \sim \rho} \Exp_{\z \sim \mu^m} \E_\mu\left( w_{D}(\z)\right) & \leq \frac{1}{T}\sum_{t=1}^{T} \E_{\z_t}(w_{D}(\z_t))
+ 2M\left( \left\Vert D\right\Vert _{\infty }^{1/2}\right) \sqrt{\frac{%
\left\Vert D\right\Vert _{1}\left\Vert C\right\Vert _{\infty }}{m}}~~~~~~~~\\
& 
~~~~~~~~+7L   \max_{D\in {\cal D}} \{\trace D \}\sqrt{\frac{\ln \left( 2mT\right)
\Vert \hat{C}\Vert _{\infty }}{T}}+\sqrt{\frac{\ln \left(
1/\delta \right) }{2T}},
\end{align*}
where and $C$ and $\hat{C}$ are respectively the true and empirical total
covariance operator for the data.
\end{theorem}

The proof uses the following theorem to bound the Gaussian complexity in the
exactly the same way as Theorem 7 is used to prove Theorem 2 in \cite{Mau9}%
.

\begin{theorem}
	Suppose $f:\left( H\times \left[ 0,1\right] \right) ^{m}\rightarrow \left[
	0,1\right] $ satisfies the Lipschitz condition%
	\[
	f\left( \mathbf{x},\mathbf{y}\right) -f\left( \mathbf{x}^{\prime },\mathbf{y}%
	\right) \leq \frac{L}{m}\left\Vert \mathbf{G}\left( \mathbf{x}\right) -%
	\mathbf{G}\left( \mathbf{x}^{\prime }\right) \right\Vert _{Fr},
	\]%
	for all $\mathbf{x},\mathbf{x}^{\prime }\in H^{m}$ and all $\mathbf{y}\in %
	\left[ 0,1\right] ^{m}$. Let $\mathcal{D}$ be a class of nonnegative
	definite operators on $H$. Fix a meta-sample $\left( \mathbf{X},\mathbf{Y}%
	\right) =\left( \left( \mathbf{x}^{1},\mathbf{y}^{1}\right) ,...,\left( 
	\mathbf{x}^{n},\mathbf{y}^{n}\right) \right) $. Then with%
	\[
	\mathcal{F}=\left\{ \left( \mathbf{x},\mathbf{y}\right) \mapsto f\left(
	D^{1/2}\mathbf{x},\mathbf{y}\right) :D\in \mathcal{D}\right\} 
	\]%
	we have 
	\[
	\Gamma \left( \mathcal{F},\left( \mathbf{X},\mathbf{Y}\right) \right) =\frac{%
		2}{T}\mathbb{E}_{\gamma }\sup_{f\in \mathcal{F}}\sum_{t=1}^{T}f\left( 
	\mathbf{x}_{t},\mathbf{y}_{t}\right) \leq 5L\left\Vert \mathcal{D}%
	\right\Vert _{1}\sqrt{\frac{\ln \left( 2mT\right) \left\Vert \hat{C}%
			\right\Vert _{\infty }}{T}}.
	\]
\end{theorem}

\begin{proof}
	The proof follows exactly the proof of Theorem 7 in \cite{Mau9} up to the
	statement of the following inequality (in \cite{Mau9} this is equation (6))%
	\begin{equation}
	\Gamma \left( \mathcal{F},\left( \mathbf{X},\mathbf{Y}\right) \right) \leq 
	\frac{2}{T}\mathbb{E}_{\gamma }\sup_{D\in \mathcal{D}}\frac{L}{m}%
	\sum_{t=1}^{T}\sum_{i,j=1}^{m}\gamma _{ij}^{t}\left\langle
	x_{i}^{t},Dx_{j}^{t}\right\rangle .  \label{NewGaussianBound}
	\end{equation}%
	Define an operator $J_{ij}^{t}$ on $H$ by $J_{ij}^{t}z=\left\langle
	z,x_{i}^{t}\right\rangle x_{j}^{t}$. Then by H\"{o}lder's inequality%
	\[
	\sum_{t=1}^{T}\sum_{i,j=1}^{m}\gamma _{ij}^{t}\left\langle
	x_{i}^{t},Dx_{j}^{t}\right\rangle =\left\langle
	\sum_{t=1}^{T}\sum_{i,j=1}^{m}\gamma _{ij}^{t}J_{ij}^{t},D\right\rangle
	_{2}\leq \left\Vert \sum_{t=1}^{T}\sum_{i,j=1}^{m}\gamma
	_{ij}^{t}J_{ij}^{t}\right\Vert _{\infty }\left\Vert D\right\Vert _{1}.
	\]%
	Now 
	\begin{eqnarray*}
		\left\Vert \sum_{t=1}^{T}\sum_{i,j=1}^{m}J_{ij}^{t\ast
		}J_{ij}^{t}\right\Vert _{\infty } &=&\sup_{\left\Vert u\right\Vert \leq
		1,\left\Vert v\right\Vert \leq 1}\sum_{t=1}^{T}\sum_{i,j=1}^{m}\left\langle
	J_{ij}^{t}u,J_{ij}^{t}v\right\rangle  \\
	&=&\sup_{u,v}\sum_{t=1}^{T}\sum_{i,j=1}^{m}\left\langle
	u,x_{i}^{t}\right\rangle \left\langle v,x_{i}^{t}\right\rangle \left\Vert
	x_{j}^{t}\right\Vert ^{2} \\
	&\leq &m\sup_{u,v}\sum_{t=1}^{T}\sum_{i=1}^{m}\left\langle
	u,x_{i}^{t}\right\rangle \left\langle v,x_{i}^{t}\right\rangle
	=m^{2}T\left\Vert \hat{C}\right\Vert _{\infty }.
\end{eqnarray*}%
The same bound holds for the norm of the adjoint. All the $J_{ij}^{t}$
operate on the $mT$-dimensional subspace generated by the $x_{i}^{t}$, so by
the corollary to the Olivera-Tropp inequality (Corollary \ref{Corollary
	Oliveira Tropp}) with $d_{1}=d_{2}=mT\,\ $and $\sigma ^{2}=m^{2}T\left\Vert 
\hat{C}\right\Vert _{\infty }$%
\[
\mathbb{E}\left\Vert \sum_{t=1}^{T}\sum_{i,j=1}^{m}\gamma
_{ij}^{t}J_{ij}^{t}\right\Vert \leq \frac{5}{2}\sqrt{m^{2}T\left\Vert \hat{C}%
	\right\Vert _{\infty }\ln \left( 2mT\right) }.
\]%
Division by $mT$ and (\ref{NewGaussianBound}) give the conclusion.
\end{proof}

\end{appendices}

\end{document}